\documentclass{article}
\usepackage[final]{colm2026_conference}

\usepackage[utf8]{inputenc}
\usepackage{hyperref}
\usepackage{graphicx} % 
\usepackage{makecell}
\usepackage{url}
\usepackage{bbm}
\usepackage{amssymb}
\usepackage{amsmath}
\usepackage{amsthm}
\usepackage{multirow,multicol}
\usepackage{amsfonts}
\usepackage{subcaption}
\usepackage{xspace}
\usepackage{booktabs}
\usepackage{sidecap}
\sidecaptionvpos{figure}{t}

\newtheorem{definition}{Definition}
\mathchardef\mhyphen="2D

\newcommand{\eg}{\emph{e.g.,}}
\newcommand{\name}{PA~}

\newcommand{\vshat}{\ensuremath{\hat{v}_s}\xspace}
\newtheorem{theorem}{Theorem}

\usepackage{microtype}
\usepackage{hyperref}
\usepackage{url}
\usepackage{booktabs}

\usepackage{lineno}

\definecolor{darkblue}{rgb}{0, 0, 0.5}
\hypersetup{colorlinks=true, citecolor=darkblue, linkcolor=darkblue, urlcolor=darkblue}

\begin{document}

% \twocolumn[
\title{A Prior-Aware Metric for Efficiently Distinguishing Memorization from Generalization in Large Language Models}
% \author{Trishita Tiwari, Ari Trachtenberg, G. Edward Suh}
% \begin{icmlauthorlist}
\author{Trishita Tiwari 
\\ Cornell University
\And Ari Trachtenberg 
\\ Boston University
\And G. Edward Suh \\
NVIDIA, Cornell University
}
\ifcolmsubmission
\linenumbers
\fi

\maketitle
% \begin{abstract}
% This work proposes a computationally inexpensive method to measure memorization of training data in LLMs (Large Language Models) while accounting for generalization. Prior approaches such as counterfactual memorization~\cite{zhang2023counterfactual}, have been computationally expensive, and therefore only been studied in limited settings.  However, our new metric, Prior-Aware memorization, does not require training any new models, and can thus be directly applied to existing LLMs trained on large amounts of data. We evaluate our metric on two pre-trained models, Llama and OPT, trained on the Common Crawl and The Pile, respectively. We discover that for the largest models, $55-90\%$ of the sequences that would  be classified as ``memorized'' in earlier models are, in fact, generalizable sequences.
% \end{abstract}

\begin{abstract}
    Training data leakage from Large Language Models (LLMs) raises serious concerns related to privacy, security, and copyright compliance. A central challenge in assessing this risk is distinguishing prefix-specific memorization of training data from the generation of statistically common sequences. Existing approaches to measuring memorization often conflate these phenomena, labeling outputs as memorized even when they arise from generalization over common patterns. Counterfactual memorization and other related metrics \citep{zhang2023counterfactual, wang2025generalization, lesci2024causal} provide principled solutions, however, their reliance on retraining multiple baseline models or parsing through the training data makes them computationally impractical at scale.
    This work introduces \emph{Prior-Aware memorization}, a theoretically grounded, lightweight and training-free criterion for identifying prefix-specific memorization in LLMs. The key idea is to evaluate whether a candidate suffix is strongly associated with its specific training prefix or whether it appears with high probability across many IID sampled sequences from the training data distribution due to statistical commonality. We correlate our metric with counterfactual memorization, and also evaluate it on the training corpora of two pre-trained models, LLaMA and OPT. Our results show that between 55\% and 90\% of sequences previously labeled as memorized fail our criterion and are consistent with statistical commonality
    % Similar findings hold for the SATML training data extraction challenge dataset, where roughly 40\% of sequences exhibit common-pattern behavior despite appearing only once in the training data. These results demonstrate that low frequency alone is insufficient evidence of memorization and highlight the importance of accounting for model priors when assessing leakage.
\end{abstract}

% either by comparing models trained with and without a target sequence, or by computing pre-training data frequencies. 

\section{Introduction}
 The growing adoption of Large Language Models (LLMs) has amplified concerns over training-data leakage. Two prominent issues include (a) copyright and licensing violations, which have been the subject of several lawsuits and prior literature~\citep{karamolegkou2023copyright, tremblay_v_openai, Kadrey_v_Meta,nyt_v_msft, chang2023speak}, and (b) exposure of sensitive data, such as Personally Identifiable Information (PII)~\citep{carlini2021extracting, mozes2023use}. As a result, several studies have attempted to quantify LLM training data leakage, often using various metrics to
 benchmark the extraction of training data in different training and inference scenarios~\citep{carlini2021extracting, carlini2022quantifying, karamolegkou2023copyright,yu2023bag, biderman2024emergent}.

Prior approaches to quantifying memorization in LLMs often overlook the models’ capacity to generalize, conflating prefix-specific memorization with the generation of statistically common sequences~\citep{carlini2022quantifying,yu2023bag}. For example, a prompt like “The murder was committed by” may yield “John Doe” with high probability not because the model memorized this sequence from training. Even if the full sequence appears in the training corpus, the high-probability of the output may simply reflect the statistical prevalence of 'John Doe' as a common placeholder name. Recent work confirms that LLMs can reproduce training text verbatim by generalizing from related patterns rather than recalling specific examples~\citep{liu2025language}. Thus, statistically likely sequences should not necessarily be classified as \emph{memorized}.

\emph{Counterfactual Memorization}~\citep{zhang2023counterfactual} seeks to more robustly quantify such generalization by comparing the likelihood of leaking a training sequence from models trained with and without that sequence.  In this approach, a sequence has a high probability of leakage in both models is considered unlikely to be counterfactually memorized. Unfortunately, explicitly computing this difference may be very expensive and impractical for production language models, as it requires training several ``counterfactual'' or ``baseline'' models for every training sequence, all within tightly controlled settings. For this reason, prior studies on Counterfactual memorization have been limited in scale~\citep{zhang2023counterfactual,ghosh2025rethinking}. While there have been other metrics proposed to solve this issue~\citep{lesci2024causal,wang2025generalization}, they remain fundamentally constrained by the same computational intractability at the scale of modern production models as Counterfactual memorization. This motivates the need for a cheaper method for filtering out statistically common sequences.

% No one has looked at generalization at scale because its so expensive

Toward that goal, we introduce \emph{Prior-Aware memorization} (\name memorization), a theoretically grounded and efficient, inference-only alternative to Counterfactual memorization (and related metrics). Our approach
filters out statistically likely outputs (suffixes) depending on whether they are strongly associated \emph{only} with their specific training prompts (prefixes). It distinguishes prefix-specific memorization from statistical likelihood by testing whether a suffix retains high generation probability when the model is prompted with prefixes sampled IID from the training data (according to the same distribution). If a suffix appears with high confidence across many independent prefixes, its high-likelihood reflects generalization from statistical commonality rather than from memorization of a specific prefix--suffix pair.

We validate our approach empirically against Counterfactual memorization~\citep{zhang2023counterfactual}, and evaluate our methodology on text from the training corpus of two pre-trained models: Llama, and OPT. 
% Our target sequences were either (a) randomly sampled long sequences (to simulate copyright-infringement scenarios), or (b) Named Entities (\eg names of individuals, places, etc.) to simulate the risk of leaking Personally Identifying Information (PII). 
We find that using the largest OPT and LLama models, $55-90\%$ of the sequences that would have been labeled as memorized are in fact statistically common sequences. More surprisingly, we note that a similar results occur even with the SATML training data extraction challenge dataset~\citep{yu2023bag}, where around $40\%$ of sequences are ``common'' in nature; this is despite the fact that each sequence in the challenge dataset occurs only once in the entire training data (1-eidetic). \emph{Our findings highlight the significance of looking beyond verbatim production of training data when labeling a sequence as memorized, urging us to rethink some of our previous notions about memorization in LLMs}. 

%AT - wouldn't low-frequency be labeled not memorized?
%TT - sometimes models memorize something even if its low-frequency if the sequence is out-of-distribution

Note that we audit known, naturally occurring prefix-suffix pairs using likelihood
access and target-independent prefixes sampled from the training
distribution; synthetic injections are used only for validation. We do
not study adaptive prompt search or discovery of unknown records, meaning that ``prior-aware'' is a post-hoc auditing metric, not an adversarial-extraction bound.

% \subsection{Roadmap}
The remainder of this work is organized as follows: in Section~\ref{sec:motivation}, we outline the  research gaps motivating our work. Section~\ref{sec:metric} includes the definitions of the metrics we use for our experimental results. We present our findings in Section~\ref{sec:eval} and highlight our novel conclusions. We conclude this work in Section~\ref{sec:conclusion}.

\section{Generalization vs Memorization: Issues with Leakage Measures}
\label{sec:motivation}

%AT - please check the logic below, as it is critical to the work.
Prior work generally considers training data that can be produced verbatim with high probability to be ``memorized''~\citep{carlini2021extracting, carlini2022quantifying, biderman2024emergent, yu2023bag, shi2024detecting}. However, as noted in our Introduction, such methods may misclassify statistically common sequences as memorized. To make this limitation precise, consider a prefix $p$ and suffix $s$, with $p \Vert s \in D$ indicating that their concatenation occurs in the dataset $D$. One popular instantiation of existing memorization metrics typically deem the sequence $p \Vert s$ to be \emph{memorized} if the conditional probability $P(s \mid p)$ is large~\citep{carlini2022quantifying,yu2023bag}. Following prior work, we refer to this definition as \emph{extractable memorization}, because the memorized content can be elicited from the model by supplying prefix $p$. However, we argue that memorization should not be defined this way. This is because there are multiple reasons that $P (s\mid p)$ can be high, as can bee seen through Bayes' rule:
\begin{equation}
\label{eq:bayes}
P(s \mid p) = \frac{P(p \mid s) \cdot P(s)}{P(p)}  
\end{equation}

Indeed, $P(s \mid p)$ can be high if:
\begin{enumerate}
    \item If $P(s)$ is very large, suggesting that $s$ may be statistically popular (generic). For instance, in the sequence if $p$~=~``The murder was committed by'' and $s$~=~``John Doe'', $P(s)$ may be large  because $s$  is a popular occurrence in the dataset. 
        
    \item If the relative belief ratio $\frac{P(p \mid s)}{P(p)}$ is large. This would happen if $P(p)$ is small (\eg the prefix is a rare prefix in the dataset) relative to $P(p \mid s)$, suggesting that the suffix is a strong indicator of the prefix, as we would expect for memorization.
\end{enumerate}

We argue that if $P(s|p)$ is high due to the former reason, then sequence $p\Vert s$ is \emph{not memorized} by the model. Indeed, variations of this observation have been made by other works~\citep{zhang2023counterfactual, lesci2024causal, prashanth2024recite}, and have  led to alternative definitions of memorization that we discuss below.

% The converse of the definition---i.e., a memorized sequence will have a high $P(s|p)$ \emph{is} reasonable under maximum-likelihood training, which seeks to maximize $\prod_{(p,s) \in D} P(s \mid p)$, since a training example being memorized would imply a high $P(s \mid p)$
% Similarly, concurrent work ~\citep{morris2025much} also defines a compression based memorization metric using Kolmogorov complexity. 
%However, their method relies on an ``all-knowing'' reference model that embodies the true data distribution, and so it is unclear if this is feasible for practical purposes.

A more rigorous metric that captures this argument is \emph{Counterfactual Memorization} \citep{zhang2023counterfactual}, which quantifies the performance differential of a model trained with versus without a specific target sequence. Under this definition, a sequence is considered memorized only if it exhibits low performance on models from which it was excluded; conversely, high performance across both model variants indicates generalization, as the output is not contingent on explicit training exposure. While robust, this approach is often computationally prohibitive. One may consider cheaper estimates to this calculation, such as training approximate baselines with a different data distribution; however, it is fundamentally difficult to make these estimates accurate~\citep{zhang2024membership}.

Alternative metrics, such as \emph{Distributional Memorization} \citep{wang2025generalization}, attempt to account for generalization by correlating output probabilities with pretraining frequencies via a ``task-gram''. However, constructing such a task-gram at the scale of modern production corpora remains largely infeasible. Similarly, \cite{lesci2024causal} estimate causal memorization profiles from training checkpoints. Contextual Memorization \citep{ghosh2025rethinking} is instead grounded in local overfitting: it compares training recollection with the best leave-one-out recollection across epochs and is stricter
than counterfactual memorization. Exact computation requires per-target leave-one-out retraining and access to the training trajectory. PA uses a Bayes-motivated within-model likelihood ratio, trading that counterfactual, training-dynamics interpretation for inference-only scalability.

% Similarly, \cite{lesci2024causal} and \cite{ghosh2025rethinking} propose a causal alternative by comparing performance across model checkpoints flanking the sequence’s introduction during training. Yet, the efficacy of this method is limited by the temporal resolution of standard training regimens; since models are typically checkpointed only every few hundred or thousand steps, isolating the causal influence of a single training instance remains an open challenge.

There have also been other definitions of memorization where it is unclear if they can practically distinguish sequences $p \Vert s$ with statistically popular targets $s$ from their truly memorized counterparts. One such definition is by \cite{schwarzschild2024rethinking}, which classifies a sequence as memorized if it can be generated with a compressed prompt. Similarly, concurrent work by \cite{morris2025much} defines a compression-based memorization metric using Kolmogorov complexity (which is not computable in general).   

Failing to account for generalization capabilities may greatly overestimate memorization, as indicated by a recent study~\citep{liu2025language}. This emphasizes the clear need for an inexpensive, but accurate, method for identifying whether a particular training sequence is truly memorized, or just statistically common.

\section{Our Metric: Prior-Aware (PA)  Memorization}
\label{sec:metric}

We next propose and describe our approach to quantifying \emph{Prior-Aware} [ \name] memorization, starting with a definition and proceeding to discuss how it can be feasibly computed. 
%
% We empirically test our metric's ability to filter out generic sequences by measuring its correlation with Counterfactual Memorization from the literature~\citep{zhang2023counterfactual}.
%
Throughout this section, we presuppose a language model $M$, defined as follows.
\begin{definition} The language model $M$ maps an input token sequence 
$p = (t_1, t_2, \ldots, t_j)$ to a distribution for the next token $t_{j+1}$:
\[
M : p \mapsto P(\cdot \mid p).
\] 
\end{definition}

We assume that the model $M$ is trained on a dataset $D$, consisting of sequences of prefix and suffix pairs $p \Vert s$ (where $\Vert$ is the concatenation operator). We will use the notation $P(x|y;M)$ to denote the probability of generating sequence $x$ from such an application of $M$ to an initial sequence $y$, and $P(x;M)$ will denote the total probability of seeing $x$ from any starting sequence.  We also assume that text generation proceeds auto-regressively from a starting sequence $y$, and each time generating (and incorporating into $y$) a subsequent token. 
%AT - is this what you intended?
%TT - thank you! This looks great! Just made a minor modification from (p,s) to p||s to keep consistent with Sec 2.
Where the choice of model is clear, we may omit it for brevity, and write simply $P(x|y)$ or $P(x)$.

\begin{definition}[Prior-Aware memorization]
\label{def:ps-mem}
    For tuning thresholds $0 \leq m \leq 1$ and $n \geq 0$ we say that a sequence $p\Vert s \in D$ is \emph{Prior-Aware memorized} by a model $M$ if $P(s \mid p; M) > m$ and $\frac{P(s \mid p; M)}{P(s; M)}>n$.
\end{definition}

\textbf{Extractability and Relative Belief.}  
The definition imposes two requirements: First, that $P(s \mid p; M) > m$ (see Section~\ref{sec:prob-leak}), which ensures that the suffix $s$ has a high probability of being generated verbatim by model $M$ when prompted with $p$. This has been adopted from prior work~\citep{carlini2022quantifying}. %AT - I'm not sure what is meant by the "consistence" here
%AT  -I fixed the typoe beleif -> belief
%
Second, the definition insists that the relative belief ratio, $\frac{P(s \mid p; M)}{P(s; M)} > n$ (see Section~\ref{sec:est-gen}), implying that $s$ is a strong indicator of $p$. A large ratio ($n \gg 1$) suggests that the presence of $p$  predicts the presence of $s$, or, in other words, $s$ much more likely following $p$ than in general.

\textbf{Rationale for the Dual Criteria.}
Although the relative belief ratio incorporates the conditional probability $P(s \mid p; M)$, we explicitly maintain both terms for practical reasons. First, evaluating the conditional probability is significantly more tractable than estimating the marginal probability $P(s; M)$ (see Sections~\ref{sec:prob-leak} and ~\ref{sec:est-gen}). Thus, in practice, the first condition serves as an efficient filter: if a sequence fails the $m$-threshold for extractability, it is discarded before performing the more expensive marginal density estimation required for the second term. Conceptually, if $p \Vert s$ lacks a sufficient probability of verbatim generation, its status as a statistically common  sequence is moot, as it does not pose a measurable risk of extractablility.

% Note that we require two separate terms $P(s \mid p; M) > m$, \emph{and} $\frac{P(s \mid p; M)}{P(s; M)} > n$. One may note that the latter term already contains the former, however, computing the former term is much more computationally feasible (See Section~\ref{sec:prob-leak}). Thus, in practice, this can be used as an effective filter that obviates the need for computing the latter term if the former term is below the threshold. Intuitively, we argue that if a sequence $p \Vert s$ does not have a high enough probability of being generated verbatim, we do not care if it's suffix is statistically common or not since the sequence is not memorized anyway.

%A large ratio arises when $P(s; M)$---the prior probability reflecting how ``common'' or ``popular'' $s$ is---remains small while $P(s \mid p; M)$ is large. In this case, the model generates $s$ with high probability only when conditioned on its training prefix $p$ \emph{without being statistically likely}, indicating memorization rather than generalization.

To compute our Prior-Aware metric for any given pair $p \Vert s \in D$, we need to compute two terms: $P(s\mid p; M)$ and $P(s; M)$.  We explain how to do this in the next two subsections.
%To simplify notation, we omit the dependence on the model $M$ in probability expressions; for example, $P(s)$ denotes $P(s; M)$ unless otherwise specified.

% We first describe the metric for probabilistic leakage and then our methodology for estimating counterfactual memorization. 

% For the sake of brevity, we also define abbreviations for each of these metrics to use in subsequent sections.

\subsection{\texorpdfstring{$P(s\mid p; M)$}{P(s|p; M)}: verbatim generation of \texorpdfstring{$s$}{s} given \texorpdfstring{$p$}{p}}
\label{sec:prob-leak}

% \paragraph{Preliminaries}
% Like in Section~\ref{sec:background}, we use a model $m_{\mathcal{V},\theta}$. $m_{\mathcal{V},\theta}$ has a vocabulary $\mathcal{V}$ and is parameterized by $\theta$. Its output logit vector is normalized into a probability distribution by a normalization function $n_\psi$ (e.g., softmax). After this, a token is sampled via a decoding function $d_\phi$ (e.g., top-$k$, top-$p$, etc). For brevity, we define the following composition of the above functions $M_\alpha = d_\phi \circ n_\psi \circ m_{\mathcal{V},\theta}$ where $\alpha = \{ \mathcal{V}, \theta, \psi, \phi \}$. 
% $M_{\alpha}$ thus represents a model $m_{\mathcal{V},\theta}$ with a fixed normalization function $n_\psi$ and decoding scheme $d_\phi$.

% \begin{definition}
% Token Probability: The probability of producing a token $t_{j+1}$ when the model is prompted with a sequence of $j$ tokens $t_{1:j}$ is defined as by \emph{Token Probability (TP)}:
% \begin{align}
% \begin{split}
%     TP(t_{j+1},t_{1:j},M_\alpha) &= P(t_{j+1}|t_{1:j};M_\alpha) \\ &= M_\alpha(t_{1:j})[t_{j+1}]
%     \label{eq:tlp}
% \end{split}
% \end{align}    
% \end{definition} 

For notational convenience, we write $t_{1:j}$ to denote the sequence of tokens 
$(t_1, t_2, \ldots, t_j)$. Using this notation, the probability of generating the
$k$-token suffix $s = t_{j+1:j+k}$ when prompted
with the $j$-token prefix $p = t_{1:j}$ derives from the chain rule:
\begin{align}
P(s \mid p) = \prod_{i=j+1}^{j+k} P(t_i \mid t_{1:i-1}).
\label{eq:slp}
\end{align}

% Moving forward, we omit the subscript $M_\theta$ from $P(s \mid p)$ for the sake of brevity and use $P(s|p)$ to refer to the probability from a model with fixed parameters. 

Equation~\ref{eq:slp} thus provides a closed form for the probability of leaking our target sequence verbatim from a given prefix. In practice, this term can be efficiently computed in a single forward-pass per token~\citep{tiwari2024sequence,hayes2024measuring}, and we can thus
%, since models trained with the language modeling objective already output a probability distribution over the next token for a given prompt. 
 use this to compute whether $P(s \mid p) > m$ for a given $p \Vert s$. 

% Next, we estimate which of these high-probability sequences are generic, as described below. 

\subsection{\texorpdfstring{$P(s; M)$}{P(s; M)}: generality of $s$}
\label{sec:est-gen}

The second component of our definition is $\frac{P(s \mid p)}{P(s)}$.  To compute the denominator, we would like to calculate the total probability over all prefixes $V^*$ in the model's vocabulary $V$.   To do this, let \(\pi_{\mathcal V^*}^{D}\) denote the empirical distribution over
valid prefix occurrences induced by the sampling procedure on \(D\).  Then,
\begin{equation}
P(s;M)\;\hat{=}\; v_s
    =\mathbb{E}_{p\sim\pi_{\mathcal V^*}^{D}}
      [P(s\mid p;M)]
    =\sum_{p\in\mathcal V^*}
      \pi_{\mathcal V^*}^{D}(p)P(s\mid p;M) \\
%    =
%    P(s)~\hat{=}~\vs = \sum_{p_i \sim V^*} P(s \mid p_i) \cdot P(p_i)
\end{equation}

However, this calculation is computationally prohibitive.  Instead, we use an unbiased estimator \vshat (inspired by Monte-Carlo Integration~\citep{binder1992monte}) that converges to the true probability $P(s)$ as the number samples $c$ grows to infinity.

%AT - Monte Carlo - I don't think this estimator is "new", and this should probably be stated.
% TT - Monte Carlo integral?

\begin{definition}
\label{def:estimator}
Given samples $q_1, q_2, \ldots q_{c}\overset{\mathrm{i.i.d.}}{\sim}
\pi_{\mathcal V^*}^{D}$, our \emph{total probability estimator} is given by
\begin{equation}
    \vshat = \frac{1}{c}\sum_{i = 1}^{c} P(s \mid q_i)
\end{equation}
\end{definition}

Note this estimator requires samples $q_i$ to match the distribution of prefixes $p_i \sim \pi_{\mathcal V^*}^{D}$ and be statistically (but not necessarily semantically) independent,
and it is thus computed for a specific model $M$; for brevity, we omit $M$ from the notation where it can be straightforwardly inferred. The next two theorems show that $\hat{v}_s$ is unbiased.

\begin{theorem}
% Over several random trials, %AT - I'm not sure this preamble is needed.
$\mathbb{E}[\hat{v_s}] = v_s$.
\label{the:expect}
\end{theorem}
\begin{proof}
By linearity of expectation, $\mathbb{E}[\hat{v_s}] = \frac{1}{c}\sum_{i=1}^{c}    
\mathbb{E}[P(s \mid q_i)]$.
Since $q_i$ are chosen i.i.d, we have that $\mathbb{E}[P(s~\mid~q_i)] = \mathbb{E}[P(s \mid q_1)]$.  This means that
$\mathbb{E}[\hat{v}_s] = \frac{1}{c} \cdot (c \cdot \mathbb{E}[P(s~\mid~q_1)]) = \mathbb{E}[P(s~\mid~p_1)] = v_s$
% &= E[P(s \mid q_1)] \\
% &= E[P(s \mid p_1)] = v_s
Recall that $q_1$ and $p_1$ are identically distributed to $p_i$.
\end{proof}

Indeed, the variance $\lim_{c \rightarrow \infty}Var[\hat{v_s}] = 0$, as demonstrated in the following theorem (whose proof is relegated to Appendix~\ref{app:proof-var} due to space constraints):
\begin{theorem}
\label{the:var}
$Var[\hat{v_s}] = \frac{1}{c} \cdot Var[P(s \mid p_i)] \leq \frac{1}{4c}$ 
\end{theorem}
In other words, over a large number of prefixes sampled from the dataset, the error of the estimator tends to $0$.

% TODO:
% Cite usage of near duplicates to indicate simplicity
% Cite 
\section{Evaluation}
\label{sec:eval}

In this section, we start by empirically evaluating the correlation of Prior Aware memorization with Counterfactual memorization~\citep{zhang2023counterfactual}, an existing metric that aims to distinguish memorized $p\Vert s$ from those where $s$ is statistically common. Due to the prohibitive cost of computing Counterfactual memorization, we first conduct two small-scale correlation experiments with real and synthetic data. We discuss the experiments with real data in Section~\ref{sec:counterfactual}, but relegate a discussion of the latter to Appendix~\ref{app:syn}. We then evaluate PA memorization alone at scale with pretrained models and note any trends in Section~\ref{sec:large-scale}.

\subsection{Correlation with Counterfactual Memorization}
\label{sec:counterfactual}

We first empirically show that \name memorization can indeed distinguish $p\Vert s$ where $s$ is statistically popular by measuring its correlation  with Counterfactual memorization~\citep{zhang2023counterfactual}. To do so, we train several models in a controlled setting, and measure both Counterfactual memorization and our Prior-Aware memorization $\frac{P(s \mid p)}{\vshat}$. We then show that these quantities are positively correlated with a Pearson Correlation Coefficient of 0.71. 

\subsubsection{Counterfactual Memorization}
For completeness, we restate the definition of Counterfactual Memorization from \citep{zhang2023counterfactual}.

\begin{definition}[Counterfactual Memorization~\citep{zhang2023counterfactual}] 
\label{def:counterfactual-mem}
Given a training algorithm $A$ that maps a training
dataset $D$ to a trained model $f$, and a measure $L(f, x)$ of the accuracy of $f$ on a specific example
$x$, the Counterfactual Memorization of a training example $x$ in $D$ is given by
\vspace{-0.3in}
\end{definition}
\begin{align}
    \text{mem}(x) \triangleq \underset{\scriptscriptstyle S \subset D, x \in S}{\mathbb{E}} [L (A(S), x)] - \underset{\scriptscriptstyle S\prime \subset D, x \notin S^\prime}{\mathbb{E}} [L (A(S^\prime), x)]
\end{align}
where $S$ and $S^\prime$ 
%AT - Are we sure the correction of S' above is right?
%TT - I'm pretty sure they meant S', at least based on the text that they have explaining the equation
are subsets of training examples sampled from $D$. The expectation is taken with
respect to the random sampling of $S$ and $S^\prime$, as well as the randomness in the training algorithm $A$. Generally speaking, the first term measures expected accuracy over models that contain the target sequence $x$, and the second over models without $x$. A sequence is only considered Counterfactually memorized if this difference is large, as it indicates that the high performance on $x$ by the former models is due to explicit training exposure.

\subsubsection{Metric measurements.} We describe how we measure Definitions~\eqref{def:counterfactual-mem} and~\eqref{def:ps-mem} in our experiments. 

% \begin{enumerate}
% AT - this next paragrah is really not clear ... is M trained on exact and M' on near-dup?
\textbf{Counterfactual Memorization.} In our setup, we interpret $x$ in~\eqref{def:counterfactual-mem} as $p \Vert s$, where $p$ and $s$ are of equal length.  To simplify notation, we use $M$ and $M^\prime$ to refer to $A(S)$ and $A(S^\prime)$ respectively, which we refer to as the target ($M$) and baseline ($M^\prime$) models (described below). To measure the accuracy of the model $M(x=p\Vert s)$, we utilize the measure $\log(P(s\mid p; M))$, following the definition of extractable memorization based on~\citep{carlini2022quantifying}. 
% The expectation is taken over all models trained with the same ratio of near-duplicates to exact copies.
% AT - I don't understand what is informly distributed:
% TT - Ah I mean that the models are all equally likely, but I think this sentence is confusing so I'll remove it!

% As in~\citep{zhang2023counterfactual}, we assume these models are uniformly distributed, so the expectation reduces to the average of $\log(P(s \mid p))$ across multiple trained models. 

Thus, we measure the following:
\begin{align}
     \mathbb{E}_M[\log(P(s\mid p; M))] - \mathbb{E}_{M^\prime}[\log(P(s\mid p; M^\prime))] 
     \label{plot:x-axis}
\end{align}

\textbf{\name Memorization.} While \name memorization (Definition~\ref{def:ps-mem}) is formally defined as a binary indicator, Counterfactual memorization yields a continuous-valued score. To facilitate a direct comparison between these definitions, we employ the continuous relative belief ratio, $\frac{P(s \mid p; M)}{\hat{v}_{s; M}}$, omitting the discrete threshold $n$. Furthermore, although Definition~\ref{def:ps-mem} includes a preliminary conditional probability threshold, we waive this filtering step in our small-scale analysis as its primary function is to optimize large-scale experiments. This simplification allows for a direct correlation analysis between the singular PA score and the corresponding Counterfactual score for every sequence. For numerical stability and alignment with common evaluation practices, we compute this ratio in log-space:
\begin{align}
\mathbb{E}_M[\log\left(\frac{P(s\mid p; M)}{\hat{v}_s}\right)] = \mathbb{E}_M[\log(P(s \mid p; M))] - \mathbb{E}_M[\log(\hat{v}_{s;M})] 
     \label{plot:y-axis}
    \end{align}
Note that \name memorization only relies on $M$, thus obviating the need for training baseline models $M^\prime$ like with Counterfactual memorization.

% \name Memorization provides a binary indicator of whether $p \Vert s$ is memorized or not, as opposed to the continuous-valued score that Counterfactual Memorization provides. Hence, to keep the two definitions comparable, we use the raw ratio $\frac{P(s \mid p; M)}{\hat{v}_{s; M}}$ directly without the threshold $n$. Note that even though Definition~\ref{def:ps-mem} consists of two terms, we omit the former term as its primary purpose is to act as a filter in large scale experiments, which is not necessary in this small scale setting. This omission also allows us to directly measure the correlation as we are now comparing a single score from PA Memorization to the one from Counterfactual Memorization under the same setting for every sequence. Finally, we measure the ratio in log space, like so:
% compute it in log-space,  instead of measuring $E_M[\frac{P(s \mid p; M)}{\hat{v}_{s; M}}]$, we

\subsubsection{Target and Baseline Datasets.}
\paragraph*{Dataset} Here, we train $124M$-parameter GPT-2 models~\citep{radford2019language} on WikiText~\citep{merity2016pointer}. Crucially, we select naturally occurring 
$p\Vert s$ pairs from WikiText in which the suffix $s$ spans a wide range of frequencies. Specifically, we focus on pairs where 
$s$ is a named entity, following prior work on memorization~\citep{lukas2023analyzing}. Figure~\ref{fig:ne-dist} presents a frequency histogram of named entities in WikiText: for example, ``United States of America'' appears approximately 500 times, whereas ``Starlicide'' occurs only once.

We sample a total of $50$ $p \Vert s$ pairs, with $s$ drawn uniformly across the frequency buckets shown in Figure~\ref{fig:ne-dist}.  Each target requires a separately retrained leave-one-out baseline (50 baseline runs per seed here, plus the shared target model), so the cost scales linearly with the audit set. We therefore treat this experiment as a proof of concept motivating PA's inference-only scalability, not as a corpus-level prevalence estimate.

\begin{figure}
    \centering
    \includegraphics[width=0.7\linewidth]{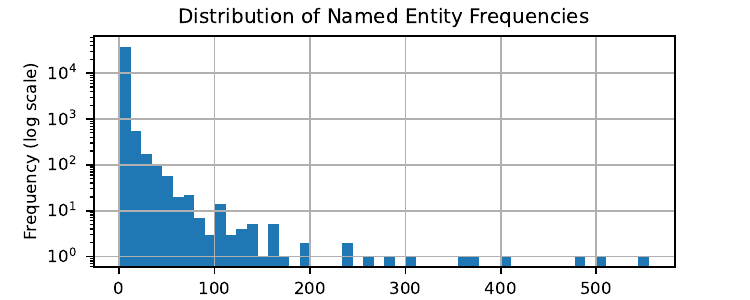}
    %\vspace{-0.1in}
    \caption{Distribution of Named Entities across Wikitext. The x-axis shows the frequency of each Named-Entity, and the y-axis shows how many named entities fit in each bucket.}
    % Fix axes
    \label{fig:ne-dist}
    %\vspace{-0.15in}
\end{figure}

\paragraph*{Target Model} % \subparagraph{}
Here, we train the model as-is on Wikitext without modifying the dataset.

\paragraph*{Baseline Model}
In this setting, we retrain the model after removing the target sequence $p \Vert s$ from the WikiText training data. For example, if $p \Vert s$ corresponds to ``I live in the'' $\Vert$ ``United States of America,'' this exact prefix–suffix pair is excluded prior to training. Importantly, we do not remove other occurrences of the suffix $s$ (e.g., ``United States of America'') that appear with different prefixes, since our goal is to measure memorization of the specific sequence $p \Vert s$ rather than the suffix in isolation.

\subsubsection{Correlation Results}
Figure~\ref{fig:wiki-counterfact-approx} shows a positive association across the 50 pairs sampled from WikiText
(Pearson \(r=0.71\); Spearman \(\rho=0.65\),
\(p=3.53\times10^{-7}\)) between Counterfactual Memorization metric (Equation~\ref{plot:x-axis}) on the x-axis, and $\frac{P(s|p)}{\vshat}$ (Equation~\ref{plot:y-axis}) on the y-axis. The mild flattening at high counterfactual
scores is consistent with exposure increasing both
\(P(s\mid p;M)\) and \(\widehat v_s\), thereby compressing their
log-ratio. PA is therefore best viewed as a ranking proxy rather than
a linearly calibrated replacement for counterfactual memorization.

% Figure~\ref{fig:wiki-counterfact-approx} also shows the Counterfactual Memorization metric (Equation~\ref{plot:x-axis}) on the x-axis, and $\frac{P(s|p)}{\vshat}$ (Equation~\ref{plot:y-axis}) on the y-axis. Here, each of the data points represents the correlation between the \name memorization and Counterfactual memorization for each of the $50$ $p \Vert s$ that we sample from WikiText. We see that both Counterfactual memorization and \name memorization tend to correlate positively, with a Pearson Correlation Coefficient of $0.71$, and a Spearman coefficient of $0.65$ (with a p-value of $3.53\cdot10^{-7}$).

\begin{figure}
    \centering
    \includegraphics[width=0.7\linewidth]{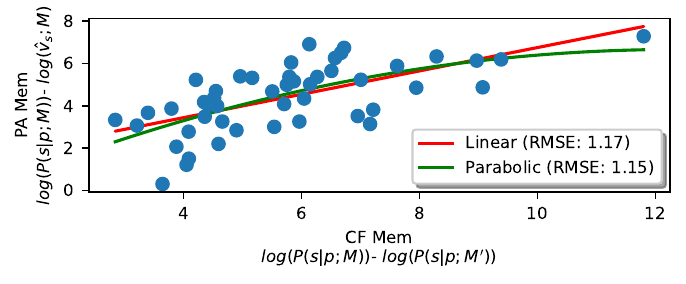}
    %\vspace{-0.1in}
    \caption{Counterfactual Memorization (x-axis) vs \name Memorization (y-axis). Each data point represents the scores of a single $p\Vert s$ in WikiText, and the PCC between the two quantities is $0.71$, and the Spearman coefficient is $0.65$ with a p-value of $3.53\cdot10^{-7}$.}
    % We see that both metrics are positively correlated across different training settings
    % Fix axes
    \label{fig:wiki-counterfact-approx}
    %\vspace{-0.2in}
\end{figure}

\subsection{Large Scale Evaluation}
\label{sec:large-scale}
%\vspace{-0.1in}
Now that we have shown PA memorization correlates with Counterfactual memorization, we now evaluate PA memorization by itself on $2$ recent generative pre-trained models with open source datasets: Llama~\citep{touvron2023llama} and OPT~\citep{zhang2022opt}. These are respectively pre-trained on the Common Crawl and The Pile. We experiment with all available sizes of the models, namely 3B, 7B, and 13B for Llama, and 125M, 350M, 1.3B, 2.7B, 6.7B and 13B for OPT. Unless otherwise stated, the default model size for which results are displayed is OPT 6.7B and Llama 7B.

\subsection{Target Sequences for Extraction}
%\vspace{-0.06in}
To compute $P(s)$ for each sequence we test, we prompt the model with $5000$ randomly sampled sequences from their respective datasets and measure the likelihood of generating $s$ in each case. We repeat this for $5$ trials. Our evaluation is performed in $3$ different settings: 

\textbf{Named Entities.} Following~\cite{lukas2023analyzing}, we randomly sample $\approx5000-8000$ sequences that contain Named Entities (NEs) from each of the datasets. We do this since NEs can simulate sensitive data that might be targets for extraction since they are typically names of individuals, organizations, places, etc. We must note that this is simply to emulate a realistic use-case; in reality, model developers would repeat the measurements with sequences that they consider sensitive. For consistency, each of our sequences consists of a prefix of $50$ tokens followed by a $4$-token long Named Entity. While we experiment with longer prefix lengths of upto $400$ tokens, the default prefix length is $50$ tokens.

\textbf{Long Sequences.} We also simulate extraction of longer sequences, which could be useful to measure the risk of leaking copyrighted data. For this, we randomly sample $5000$ sequences with $50$-token prefixes and $50$-token suffixes from each dataset, similar to prior work~\citep{biderman2024emergent,carlini2022quantifying}. Here too, we experiment with longer prefix lengths of upto $400$ tokens, but unless otherwise stated, the default prefix length is $50$ tokens.

\textbf{SATML Challenge.} Lastly, we also perform a smaller-scale, additional analysis on the dataset in~\cite{yu2023bag}, which consists of $1$-eidetic sequences (sequences where each $p \Vert s$ is known to occur only once in The Pile). This dataset was released for the 2023 SATML training data extraction challenge and consists of $15,000$ sequences, each with a prefix of $50$ tokens followed by a suffix of $50$ tokens. We present results on $1000$ of the $15,000$ sequences.
    %Again, like before, we experiment with longer prefixes of upto $200$ tokens.

\subsection{Hyper-parameters}
%\vspace{-0.05in}
The only parameters that need to be tuned in Definition~\ref{def:ps-mem} are $m$ and $n$:

\textbf{$\pmb{m}$ --- threshold for $\pmb{P(s|p)}$:} This parameter allows us to pick which $p \Vert s$ have a high risk of leakage, and thus can be adjusted based on the risk tolerance. Since $m$ is a threshold for $P(s|p)$, $\frac{1}{m}$ indicates, on average, the number of times a user would need to prompt the model with $p$ to leak the target $s$. We set $m=0.01$ for $4$-token $s$, and $m=0.0001$ for $50$-token $s$. We deliberately pick small values of $m$ to be conservative in our estimation of memorization.

\textbf{$\pmb{n}$ --- threshold for $\pmb{\frac{P(s\mid p)}{\hat{v}_s}}$:} The second parameter, $n$, allows us to differentiate $p \Vert s$ where $s$ has a large prior from those that do not. To estimate this, we compute $\frac{P(s\mid p)}{\hat{v}_s}$ over $s$ that are known to be easy-to-predict for LLMs (examples provided in Appendix~\ref{app:generic-seq}). We set $n$ to be the average of $\frac{P(s\mid p)}{\hat{v}_s}$ over these sequences. Note that since $\frac{P(s\mid p)}{\hat{v}_s}$ is a model specific quantity, the threshold is recomputed for each model tested; it does not guarantee a fixed false-positive rate. We also perform an ablation over different $n$ in Appendix~\ref{app:sensitivity}.

\subsection{Effect of Model Size on \name Memorization}
In Figures~\ref{fig:size} and~\ref{fig:challenge}, we plot both the number of \emph{extractably memorized} (simply high $P(s\mid p)$, as defined by~\cite{carlini2022quantifying} and in Section~\ref{sec:motivation}) and \emph{\name memorized} sequences as a function of model size. We also plot \name memorized sequences as a proportion of the total extractably memorized sequences (labeled ``Ratio'' in the legend). There are two interesting observations to be made: 

% \begin{enumerate}
\textbf{Gap in extractable and \name memorized sequences}. Consistent with prior work~\citep{hayes2024measuring,carlini2022quantifying,schwarzschild2024rethinking}, both extractable and \name memorization increase as the model size increases. However, there is a difference between number of extractably memorized and \name memorized sequences. Notably, this difference is quite large when the target suffixes are $4$ token Named Entities, where as many as $90\%$ of extractable memorized samples are \emph{not} \name memorized. This suggests that most of the suffixes of extractably memorized sequences are indeed statistically common rather than truly memorized. As shown in Table~\ref{tab:scores}, many of these $s$ are indeed popular entities such as politicians, celebrities, countries, etc. More surprisingly, we note a similar result in the SATML challenge dataset in Figure~\ref{fig:challenge}, where around $40\%$ of sequences are ``common'' in nature. This is despite the fact that each sequence in the challenge dataset occurs only once in the entire training data.
    
\textbf{Decreasing proportion of \name memorized sequences.} Secondly, Figures~\ref{fig:size} and~\ref{fig:challenge} show that as model size increases, the proportion of \name memorized sequences generally \emph{decreases}. This suggests that larger models are increasingly able to reproduce verbatim text by generalizing from common or near-duplicate data, rather than through true memorization---a finding consistent with recent work~\citep{liu2025language}. 
% \end{enumerate}

\begin{figure}
\centering
\includegraphics[trim={0 0.3cm 0 0.2cm},clip]{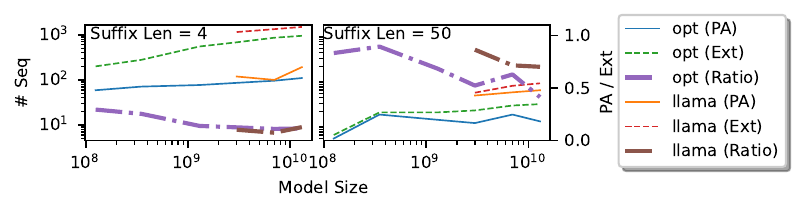}
\caption{The primary y-axis reports the number of extractable and \name memorized sequences as a function of model size for 4- and 50-token suffixes. The thicker lines denote the proportion of extractable-memorized sequences that are also \name memorized, with their scale indicated on the secondary y-axis.}
\label{fig:size}
%\vspace{-0.1in}
\end{figure}
% \begin{figure}[h!]
%     \centering
%     \begin{subfigure}[t]{0.32\textwidth}
%         \centering
%         \includegraphics[width=\linewidth]{images/counterfactual-size-breakdown-4.pdf}
%         \caption{All suffixes are $4$ token named entities}
%     \end{subfigure}
%     \hfill % Adds horizontal space between figures
%     \begin{subfigure}[t]{0.32\textwidth}
%         \centering
%         \includegraphics[width=\linewidth]{images/counterfactual-size-breakdown-50.pdf}
%         \caption{$50$ token suffixes}
%     \end{subfigure}
%     \begin{subfigure}[t]{0.32\textwidth}
%         \centering
%         \includegraphics[width=\linewidth]{images/counterfactual-size.pdf}
%         \caption{$50$ token suffixes}
%     \end{subfigure}
%     \caption{Number of a) Verbatim, and b) PS Memorized Sequences as a function of model size for two types of suffixes}
%     \label{fig:counterfactual-size-breakdown}
% \end{figure}

\subsection{Effect of Prefix Length on Discovering \name Memorization}
In Figure~\ref{fig:len} we plot the number of samples labeled as \name memorized and extractable memorized as a function of the length of prompt $p$ provided to the model. Unsurprisingly, we observe that longer $p$'s can discover more \name memorized sequences. However, note that the \name memorized of $4$ token Named Entity suffixes does not benefit as much from longer prefixes as the $50$ token suffixes. Again, this is likely due to the fact that many of these Named Entities (\eg United States of America) are popular occurrences in web-text, and likely do not need specific prompts to be produced correctly. 

% This is because Definition~\ref{def:ps-mem} only selects $p \Vert s$  that have $P(s \mid p)$ greater than a certain threshold, as one of the necessary conditions for memorization is a high probability of verbatim generation. As we make $p$ longer, more pairs meet this requirement, and thus are evaluated for being counterfactually memorized. This phenomenon has also been observed in prior literature~\citep{hayes2024measuring,carlini2022quantifying}.

\begin{figure}
    \centering    \includegraphics[width=\linewidth]{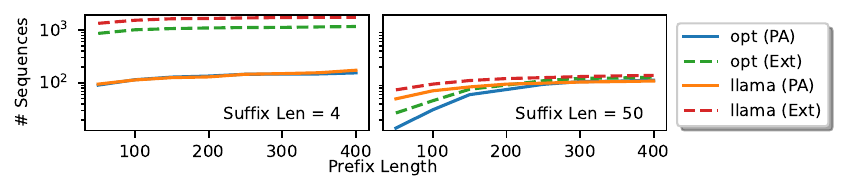}
    \caption{Number of a) Extractable, and b) \name Memorized Sequences as a function of model prefix length for two types of suffixes.}
    \label{fig:len}
    %\vspace{-0.2in}
\end{figure}

\subsection{Qualitative Analysis}
Table~\ref{tab:scores} shows a few sequences with 4-token Named Entity suffixes that have high and low \name memorization scores (i.e., $\frac{P(s\mid p)}{P(s)}$). We present more such examples in Appendix~\ref{app:examples}.  We note that all of the sequences with low scores have suffixes that are places, political figures, etc. We also attempt to verify how ``rare'' the high-scoring sequences are. While it is infeasible to determine rarity precisely within a dataset with trillions of tokens, we perform a web-search and note that the high-scoring sequences have very few search results. In fact, the highest scoring sequence in Table~\ref{tab:scores} has just a single search result, which is precisely the sequence that is included in the dataset. We also note that high-scoring suffixes do appear to be more ``artifact-like'' than their low-scoring counterparts, which appear to be more natural text. This is more apparent in longer sequences, such as the examples presented in Appendix~\ref{app:examples}.

\begin{table}[h]
% \scriptsize
\centering
\resizebox{\textwidth}{!}{%
\begin{tabular}{rl|rl}
\toprule
\makecell{\textbf{Score} \\ \textbf{(Low)}} & \textbf{Sequence ($p\Vert s$)} & \makecell{\textbf{Score} \\ \textbf{(High)}} & \textbf{Sequence ($p\Vert s$)} \\
\midrule
2.9 & \dots and is a tributary to \textbf{Saginaw Bay} & 4052  & \dots t1=``Sea Zone'' t2=`` \textbf{South Atlantic Sea Zone} \\
3.0 & \dots special prosecutor \textbf{Leon Jaworski} & 3358  & \dots misguided members of the \textbf{Autonomie Club} \\
3.7 & \dots Jack Germond and \textbf{Jules Witcover} & 2544  & \dots I'm watching Gore’s \textbf{Warmista-Fest} \\
4.0 & \dots Hospital had received \textbf{Hill-Burton} & 1560  & \dots PLUS Gold certification.- \textbf{Corsair Gold AX850}\\
\bottomrule
\end{tabular}
}
\caption{Examples of low and high scoring $p\Vert s$. $s$ in each sequence is in bold text.}
\label{tab:scores}
\vspace{-0.2in}
\end{table}

\begin{figure}[htbp]
    \centering
    % \begin{subfigure}[t]{0.52\textwidth}
    %     \centering
        \includegraphics[width=0.5\textwidth]{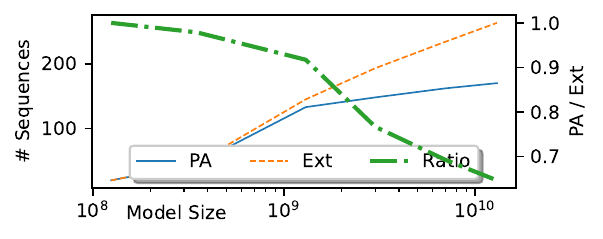}
        \caption{Memorization as a function of model size for $1$K sequences from the  SATML challenge dataset.}
        \label{fig:challenge}
    % \end{subfigure}
    % \hfill
    % \begin{subfigure}[t]{0.45\textwidth}
    %     \centering
    %     \includegraphics[width=\textwidth]{images/counterfactual-breakdown.pdf}
    %     \caption{$P(s\mid p)$ and $P(s)$ as a function of the number of exact copies of $p\Vert s$ in the training dataset.} 
    %     \label{fig:counterfact-breakdown}
    % \end{subfigure}
    % \caption{Results on the SATML Challenge Dataset (a), and $P(s \mid p)$ and $P(s)$ from Section~\ref{sec:counterfactual}}
\vspace{-0.2in}
\end{figure}

\subsection{Limitations}
\vspace{-0.05in}
In this section, we outline the limitations of PA-memorization.

\paragraph*{Sensitivity to Threshold Calibration.} The accuracy of PA memorization depends on the calibration of the relative belief ratio threshold, $n$. A model's capacity and the entropy of its training dataset directly influence baseline probabilities. Firslty, smaller, higher-perplexity models tend to yield ``flatter'' marginal distributions, which can compress the relative belief ratio and make it more difficult to cleanly separate rote memorization from generalization compared to larger, more performant models. Secondly, in high-entropy (highly diverse) datasets, specific sequences are globally rarer, lowering the average prior $P(s)$ and inherently inflating the baseline relative belief ratio $\frac{P(s|p)}{P(s)}$. Conversely, low-entropy (highly structured) datasets contain frequent boilerplate, increasing $P(s)$ and suppressing the ratio.
% \end{enumerate}

Given this, setting $n$ too high or low would lead to under- or over-flagging of memorization, respectively. To mitigate this, we dynamically compute $\frac{P(s\vert p)}{
\vshat}$ over ground-truth $s$ that
are known to be easy-to-predict for LLMs for every model size. We also provide a sensitivity analysis of the number of PA-memorized sequences to the threshold $n$ in Appendix~\ref{app:sensitivity}

% \paragraph*{Reliance on Model Perplexity.} Because PA memorization utilizes the model's own density estimation to compute the global prior $P(s)$, the metric is sensitive to the model's baseline perplexity. As described above, smaller, higher-perplexity models tend to yield ``flatter'' marginal distributions, which can compress the relative belief ratio and make it more difficult to cleanly separate rote memorization from generalization compared to larger, highly performant models.

\paragraph*{Extensions to Post-Trained Models.} Our current Monte Carlo estimator for $P(s)$ relies on sampling prefixes IID according to the original training data distribution. Consequently, our large-scale empirical validations focus on base models pre-trained directly on web text. Post-training processes, such as instruction tuning or RLHF, introduce significant distributional shifts. Applying our current sampling strategy to these post-trained models without accounting for this structural shift could bias the prior estimation. Adapting the PA framework to robustly sample and estimate priors across shifted post-training distributions remains an important direction for future work.

\subsection{Discussion: Assessing True Downstream Risks}
\vspace{-0.05in}
PA-memorization provides a quantitative lens for \textit{why} a given sequence is easily extractable. From a privacy perspective, it contextualizes statistically common PII that lacks the uniqueness required for meaningful violations, aligning with plausible deniability. In the context of copyright, PA memorization leverages the global prior $P(s)$ to filter highly generic, uncopyrightable boilerplate from true infringement. Finally, understanding this root cause is vital for technical mitigation: targeted data scrubbing or machine unlearning will largely fail if an output stems from broad statistical prevalence. By isolating rote recall from general popularity, PA memorization ensures that unlearning defenses are directed toward genuinely memorized sequences where they can actually succeed.

\section{Conclusion}
\vspace{-0.1in}
\label{sec:conclusion}
In this work, we propose \emph{Prior-Aware memorization}, an efficient, inference-only metric for distinguishing statistically common continuations from their memorized counterparts. Across evaluated models, we find that $55-90\%$ of sequences labeled as memorized by prior metrics are, in fact, statistically likely sequences that could be produced without specifically training on said sequences. This highlights how traditional metrics may overstate memorization in LLMs, and may have ramifications for privacy, copyright, and regulation.

% \section{Impact Statement}
% The aim of this work is to better analyze memorization of training data in Large Language Models, which can have use in quantifying the risk of leaking privacy-sensitive data or copyrighted content.

\bibliographystyle{colm2026_conference}
\bibliography{main}

\appendix
\section{Proof for Theorem~\ref{the:var}}
\label{app:proof-var}
\setcounter{theorem}{1}
\begin{theorem}
$Var[\hat{v_s}] = \frac{1}{c} \cdot Var[P(s \mid p_i)] \leq \frac{1}{4c}$ 
\end{theorem}
\begin{proof}
Since $p_i$ are independent, we can reduce:
\begin{align*}
Var[\hat{v}_s] &= Var[\frac{1}{c}\sum_{i=1}^{c} P(s \mid q_i)]\\
&= \frac{1}{c^2}Var[\sum_{i=1}^{c} P(s \mid q_i)] \\
&= \frac{1}{c^2} \sum_{i=1}^{c} Var[P(s \mid q_i)] \\
&= \frac{1}{c^2} \cdot c \cdot Var[P(s \mid q_1)] \\
&= \frac{1}{c} \cdot Var[P(s \mid p_i)] 
\intertext{From Popoviciu's inequality on variances,}
Var[\hat{v}_s] &= \frac{1}{c}\cdot Var[P(s \mid p_i)] \leq \frac{1}{4c} 
\end{align*}

As in Theorem~\ref{the:expect}, the transition from conditioning on $q_1$ to $p_1$ follows from the equality of their distributions.
\end{proof}

\section{Correlation Experiments with Synthetic Data}
\label{app:syn}
\textbf{Synthetic Data} We train $124$M parameter GPT-$2$ models~\citep{radford2019language} %AT - citation?
with $1000$ Wikitext~\citep{merity2016pointer} documents in two settings, as described below.  We repeat the experiments in each setting several times with different seeds and target sequences, training over $350$ different models.

%AT - you mean documents from Wikipedia?  How were they selected?  Please specify.
% TT - Oh sorry, I mean the wikitext dataset (which is a collection of wikipedia articles, very common in LLM training so I'll just cite it the link). I just sampled 1000 documents from that dataset
% TT - Link for wikitext: https://huggingface.co/datasets/Salesforce/wikitext

\begin{table}[]
    \footnotesize
    \centering
    \begin{tabular}{|c|c|}
        \hline
         Exact Copy &  Near-Dup \\ \hline \makecell{Quantum doughnuts might not exist, \\ but theoretical bakers remain hopeful.}
          & \makecell{majestic\textcolor{blue}{um} Nant\textcolor{blue}{onuts} might Conradavery \\ 258 texted \textcolor{blue}{theoretical} imperialistmlicks Shim.} \\
         \hline
    \end{tabular}
    \caption{A sequence and one possible $20\%$ near-duplicate. Matching tokens are highlighted}
    %AT - is this an 80% near-duplicate or a 20% near-duplicate?  It's a bit confusing.
    % Oops yes its 20%
    \label{tab:neardup}
    \vspace{-0.1in}
\end{table}

\paragraph*{Target Model.}
% \subparagraph{} 
In this setting, we inject ``near-duplicate'' sequences into the dataset, a strategy commonly used to study memorization versus generalization in prior work~\citep{zhang2023counterfactual,liu2025language}. 
For our purposes, near-duplicates are defined as sequences sharing $20\%$ token overlap, as illustrated in Table~\ref{tab:neardup}. 
Unlike earlier studies, which often used duplicates with much higher overlap, we intentionally adopt a lower threshold to be conservative. 
This choice is motivated by recent findings that removing high-overlap sequences improves model performance~\citep{lee2022deduplicating}. 
Thus, injecting highly overlapping sequences may not realistically reflect modern training regimes, where such examples are typically removed~\citep{zhang2022opt,touvron2023llama}. 
Our goal is to demonstrate that even with relatively low overlap, models are capable of generalizing to a given sequence without requiring many (or any) exact copies in the training data. 

%AT - 20% is actually very little overlap, not "near".
% TT - hmm yes. Good point, I tried to clarify why the overlap is so little. I wanted to show that even if there is little overlap, models can still actually learn from the little information available and generalize effectively.

%AT - I don't understand this:--from only exact copies (no near-duplicates) to only near-duplicates (no exact copies).
%AT - I don't understand the next line - what are you removing?

% In this scenario, our near duplicates are sequences which have only $20\%$  token overlap with the original sequence (Table~\ref{tab:neardup}). The sequences are injected randomly into the dataset.

To simulate varying degrees of generality, we vary the ratio of exact copies to near-duplicates of a target sequence injected into the training data. For example, an initial dataset may contain $180$ near-duplicates of a target sequence and no exact copies. In the next dataset, $30$ of these near-duplicates are replaced with $10$ exact copies, yielding $150$ near-duplicates and $10$ exact matches. The rationale is that as more exact copies are included, the model is more likely to exhibit Counterfactual memorization, predicting the target sequence because it has seen it verbatim during training. Conversely, when the dataset consists mostly of near-duplicates, the model must rely on generalization from many similar--but not identical--examples, simulating the case where a sequence is generated with high probability simply because there is a lot of similar data in the training set~\citep{liu2025language}. By carefully controlling the ratio of exact copies to near-duplicates, we can precisely modulate the extent to which the model relies on memorization versus generalization in order to reproduce the target sequence.

We use the following number of (exact copies, near-duplicates) to create $7$ distinct types of datasets: $(0, 180)$, $(10, 150)$, $(20, 120)$, $(30, 90)$, $(40, 60)$, $(50, 30)$, $(60, 0)$. In every case, the total training set size remains fixed at $1000$ sequences, with only the composition of exact versus near-duplicate instances varying. 
We repeat this process for $25$ different target sequences, training $25 \times 7=175$ different target models.

% We continue this process of reducing $30$ near-duplicates generated for every $10$ exact copies added,  until the final dataset, which contains no near-duplicates and $60$ exact copies. 
 
% AT - you mean that there are $1000$ *unique* sequences?  Because it looks like you're adding sequences to the training set each time.

\paragraph*{Baseline Model.} The Counterfactual memorization metric relies on a comparison between the target model and a baseline model. In accordance with the metric's definition of a baseline model, we remove only the exact copies of the target sequence, but keep the near duplicates in the training data. So the baseline model for a target model trained with $150$ near-duplicates and $10$ exact copies would have only the $150$ near duplicates. This is to simulate the scenario where the exact target data might be removed, but other data that can generalize to the target sequence is still present in the training dataset.

\paragraph*{Results} Figure~\ref{fig:counterfact-approx} shows the Counterfactual Memorization metric (Equation~\ref{plot:x-axis}) on the x-axis, and $\frac{P(s|p)}{\vshat}$, the component of \name memorization that measures the commonality of $s$ (Equation~\ref{plot:y-axis}) on the y-axis. Each data point represents the average over 25 different models trained. Figure~\ref{fig:counterfact-approx} illustrates that as the target sequence becomes less generic due to fewer near-duplicates, the Counterfactual Memorization metric (Equation~\eqref{def:counterfactual-mem}) increases correspondingly. As hypothesized, we also observe a strong correlation between \name memorization and Counterfactual Memorization, suggesting that our metric indeed measures generalization from similar data. 

% One interesting observation is the difference in the scale of the two metrics, a limitation that we discuss in more detail in Section~\ref{sec:lim}.

\begin{figure}
    \centering
    \includegraphics[width=0.7\linewidth]{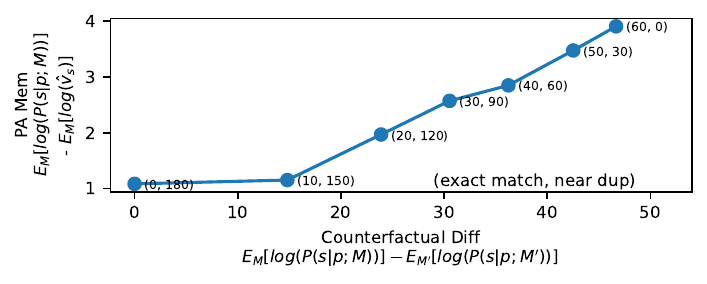}
    \vspace{-0.1in}
    \caption{Counterfactual Memorization (x-axis) vs \name Memorization (y-axis). Each data point averages over models with the same frequency of exact matches and near duplicates.}
    % We see that both metrics are positively correlated across different training settings
    % Fix axes
    \label{fig:counterfact-approx}
    \vspace{-0.15in}
\end{figure}

\section{Known Statistically Common Sequences for Threshold Calculation}
\label{app:generic-seq}
In this section, we provide a list of some short and long sequences that are known to be statistically common. We use these to calculate the threshold $n$ in Definition~\ref{def:ps-mem}. The short sequences are used to calculate the threshold for the short, $4$-token suffixes, while the long sequences are used to calculate the threshold for the longer $50$-token suffixes. In each case, each sequence is divided into equal tokens of prefix and suffix used to calculate $\frac{P(s\mid p)}{P(s)}$. Recall that $n$ is simply an average of $\frac{P(s\mid p)}{P(s)}$  

\begin{itemize}
    \item Please let me know if you have any questions,
    \item Thank you for your time and consideration,
    \item 
    I look forward to hearing from you soon,
    \item 
    If you have any concerns, please don't hesitate to contact me,
    \item Don't miss out on this limited-time offer,
    \item Our mission is to provide the best service,
    \item Learn more about our features and pricing options,
    \item Thank you for taking the time to read this message. I am writing to provide an update regarding our project and to ensure that you have all the information needed to move forward. At this stage, we have completed the initial steps and are preparing to begin the next phase. Please review the attached document, which includes a summary of progress, outstanding tasks, and anticipated challenges. If you have any questions or concerns, feel free to reach out at your earliest convenience.,
    
    \item This email is intended to confirm the details of our upcoming meeting. The session is scheduled for next week at the agreed time, and we will be covering several important items on the agenda. Please make sure to review the notes shared earlier so that we can make efficient use of our time together. If you are unable to attend, kindly let me know as soon as possible so we can reschedule or provide you with the necessary updates.
    
    \item I hope this message finds you well. I wanted to follow up on the status of the work assigned last month and check whether everything is on track for completion by the deadline. If you are facing any obstacles or require additional support, please do not hesitate to reach out. It is important for us to address any potential challenges early so that the overall timeline remains achievable.,

    \item In this article, we will discuss the basics of this topic and explain why it is important to understand this concept in the broader context of the field. We will begin with a brief introduction, followed by a step-by-step explanation of the main ideas, and finally provide some examples to illustrate how these principles can be applied in practice. Whether you are completely new to this topic or simply looking for a refresher, this guide is designed to help you gain a clear and practical understanding.,

    \item The purpose of this guide is to help you get started with the software in a straightforward and beginner-friendly way. First, we will walk through the installation process and highlight common issues that may arise during setup. Then we will cover the essential features you need to know to be productive right away. By the end of this guide, you should have a working environment and a solid foundation that will allow you to explore more advanced aspects of the software at your own pace.,
\end{itemize}

% \section{Ablation over $n$}
% \label{app:ablation}
% This is the second appendix, automatically labeled as Appendix B.

\subsection{\name Memorization Examples}
\label{app:examples}
In this section, we provide examples of $p \Vert s$ that have statistically likely $s$ and those that do not. Each sequence is preceded by it's value of $\frac{P(s\mid p)}{P(s)}$, presented in colored text. The suffix $s$ is presented in bold text. We note that low-scoring sequences have more natural-appearing english text, which is common in the training data for both OPT and Llama models. However, high-scoring sequences tend to have excessive formatting, non-english characters, URLs, are about niche topics, or are boiler-plate text (e.g., terms and conditions from companies) that appear verbatim several times on the web.

\subsection{Low \texorpdfstring{$\frac{P(s\mid p)}{P(s)}$}{P(s|p)/P(s)} sequences}
\begin{itemize}
    \item \textcolor{red}{1.48} ining the ship after this crossing." ``That's not a problem," said Emma without explanation. The purser still didn't look convinced. ``Can you read and write?" Emma would like to have told him that \textbf{she'd won a scholarship to Oxford, but simply said, ``Yes, sir." Without another word, he pulled open a drawer and extracted a long form, passed her a fountain pen and said, ``Fill this in." As Emma began}
    \item \textcolor{red}{1.48} ``pageset": "S53 476 (Tenn. 1973). Our supreme court stated the rationale for this rule: This well-settled rule rests \textbf{on a sound foundation. The trial judge and the jury see the witnesses face to face, hear their testimony and observe their demeanor on the stand. Thus the trial judge and jury are}
    \item \textcolor{red}{1.55} obvious) mistake in being unacceptably slow getting back into the play after the puck was cleared over his head into the neutral zone. He rather glided through the neutral zone as Moore passed him by to get in position to collect the loose puck. \textbf{In a cruel twist, the Caps would get one back with 2:16 left when – finally – someone was able to collect a loose puck in close and do something with it. Brooks Laich lifted a backhand over Halak almost from}
    \item  \textcolor{red}{1.68} ``Pay her off easy, easy," he screamed to the sailors struggling at the wheel, his voice barely audible over the wind. He squinted into the blinding snow in an attempt to time the turn in a smooth, a patch of \textbf{waves smaller than the average. The seas came in groups, with every seventh to ninth attaining the most prodigious heights. Sometimes these larger groups merged with a cross sea reared up from the swift tidal currents, or married together to form a rogue}
    \item \textcolor{red}{2.05} fluffy, and great). When cooking first sparked my interest in a real way, I decided I needed to master the biscuit. Not coming from a family of biscuit makers, or from a place known for excellent biscuits, I wasn't \textbf{even sure what made a great biscuit, if we are being honest. But I was on a quest for perfection.}
    \item \textcolor{red}{2.06}  The people of Thebes acclaimed Oedipus as their new king. Monster Mash The Sphinx, the Greeks said, had the head of a woman, the body of a lion, the wings of an eagle, and the \textbf{tail of a serpent. In fact, this is how the Greeks dreamed up most of their mythical creatures: by mixing and matching body parts from real animals, like a hideous, nightmarish version of Build-A-Bear Workshop. Let's take}
    \item \textcolor{red}{2.21}  up sinks: cold water Washing machine(s)  Laundry sinks: cold water Laundry sinks: hot water Dryers  Ironing facilities Food, drink and groceries
    \textbf{Fresh bread available at the camp site  Groceries: limited selection  Restaurant (with ample choice) Snack bar Takeaway meals Bar  Communal barbecue area Freezing for cooling}
\end{itemize}

\subsection{High \texorpdfstring{$\frac{P(s\mid p)}{P(s)}$}{P(s|p)/P(s)} sequences}
\begin{itemize}
%AT - what on earth is this, below?  What does WeChat have to do with this?
    % \item 的步骤。微信小程序对于登录的设计，更是用之于无形 \textbf{，在整个用户使用过程中都是无感知的，一进小程序其实}
    \item \textcolor{blue}{10.32} u{7684} u{6B65} u{9AA4}u{3002} u{5FAE} u{4FE1} u{5C0F}u{7A0B} u{5E8F} u{5BF9} u{4E8E}u{767B} u{5F55} u{7684} u{8BBE}u{8BA1} u{FF0C} u{66F4} u{662F}u{7528} u{4E4B} u{4E8E} u{65E0}u{5F62} \textbf{ u{FF0C} u{5728} u{6574} u{4E2A} u{7528} u{6237} u{4F7F} u{7528} u{8FC7} u{7A0B} u{4E2D} u{90FD} u{662F} u{65E0} u{611F} u{77E5} u{7684} u{FF0C} u{4E00} u{8FDB} u{5C0F} u{7A0B} u{5E8F} u{5176} u{5B9E}}    
    \item \textcolor{blue}{7.47}  money, now banks are closed..." ``Now I don't have cash with me..." "Anto', how much have you got?" ``Not a single lira, love." ``You see..." ``What can we do now?" ``Don't \textbf{you worry, it does not matter." ``We'll do as we have always done." ``Here is a million for you!" ``It's too much, madam." ``Thanks, madam." ``No, you must take it." ``}
    \item \textcolor{blue}{6.43} /1332) :star: - Learning Simple Algorithms from Examples. [`pdf`](http://proceedings.mlr.press/v48/zaremba16.pdf) - Learning to Trans \textbf{duce with Unbounded Memory. [`arxiv`](https://arxiv.org/pdf/1506.02516.pdf) - Listen, Attend and Spell. [`arxiv`](https://}
    \item \textcolor{blue}{4.59} be issued to the original card that was used. The refund amount will include only the amount paid by you after any discount or reward was applied to the returned item(s) and it will not include any shipping charge paid by you unless you are returning \textbf{a damaged, defective, or the wrong item was sent to you. The wedding dress is perhaps the most carefully chosen dress a woman will ever wear, and an amazing wedding ceremony could be the most unforgettable moment in a woman’s life.}
    \item \textcolor{blue}{3.99}  surface and prevent the possibility of sustaining burns unnecessarily while shaving. How can I prevent skin burns? Yes, you can. You have to adopt and adhere to the following best practices to be able to accomplish this particular feat: \textbf{Shave in the right Direction: In most instances, this problem arises when the hair is shaved against the direction of the strands thereof. To avoid this problem, it is advisable that you shave in the direction of the hair growth. This is}
    \item \textcolor{blue}{3.24} ,006 (u{5468}u{5B50}) u{3061}u{3087} u{3063} u{3068} u{FF61}| u{3069} u{3046} u{3057} u{305F} u{306E} u{FF1F}|
    u{3046} u{3093} u{FF1F}|\textbf{,011  u{4F1A} u{793E} u{306B}  u{884C} u{304B} u{306A} u{304D} u{3083} u{FF61} (u{5468} u{5B50}) u{3044} u{3044} u{304B} u{3052} u{3093} u{306B} u{3057} u{306A} u{3055}u{3044}u{FF61} u{27A1}241 00:}
    % \item 3.247930088982318 ,006 (周子)ちょっと｡ どうしたの？ うん？240 00:19:08,006 --> 00:19:13 \textbf{,011 会社に 行かなきゃ｡ (周子)いいかげんに しなさい｡➡241 00:}
    \item \textcolor{blue}{2.78} s costing to fill the car and get groceries! http://no-apologies-round2.blogspot.com/ AmericanborninCanada lol wolfie. Are you a foghorn like I am? http:// \textbf{tinyurl.com/wwsotu Wolfie ON my good days! If I went on American Idol, they’d just pay me the top prize to never come back again! However, my wife is a lovely singer and}
    \item \textcolor{blue}{2.67} , offices, and hotel). The mine-site and township were heavily polluted with asbestos and tailings from the mine were distributed around the town (5). These women have been followed up since 1973 at national and state death and cancer registries as well \textbf{as frequent questionnaires to ascertain smoking histories and demographic information. Their mortality and cancer incidence have been reported elsewhere (6-8). The studies have ethics approval from the University of Western Australia Human Research Ethics Committee. A disease cluster is an occurrence}
    \item \textcolor{blue}{2.65} ``parent":``z9iGIu1Pt",``text":"", ``mid":``z9iZtx1BP",``date":``2012-12-11 22:35:08"\},\{`` \textbf{kids":[], ``uid":``1927311555",``parent":``z9iJ1w3Q7",
    ``text":``",``mid":``z9iZjcvK8",``date":"2012-}
    \item \textcolor{blue}{2.52}  that \verb|$N^{\lambda}=0$ and let $n\in\mathbb{Z}_{>0}|
    \verb|$ be minimal such that $N^{\lambda-n\alpha}\neq 0$. Let $w\in| \textbf{\detokenize{N^{\lambda-n\alpha}$ be a non-zero element. Using the PBW Theorem, we may write $w=\sum_{i=0}^nc_if^i\overline{f}^\{}}
    \item \textcolor{blue}{2.47} , pacing around a long chart unrolled like a hound's tongue across the floor. "We've backed away from the pedigrees in the past year." Spencer's shoe brushes the edge of the pedigree chart. It is \textbf{a long genetic octopus, a family tree with arms and legs that tangle and cross, as do those of most degenerate families. Spotted throughout are symbols, keyed on the side. A dark black circle signifies Insane. A hollow}
\end{itemize}

\section{Sensitivity Analysis of Threshold $n$}
\label{app:sensitivity}

Figure~\ref{fig:sensitivity} illustrates how the number of PA-memorized sequences varies with the relative belief ratio threshold, $n$. To ensure a robust baseline, recall that we dynamically calibrate $n$ for each model size by computing $\frac{P(s\vert p)}{\vshat}$ over ground-truth sequences known to be highly predictable (refer to Appendix~\ref{app:generic-seq}). We observe that for $4$-token suffixes, the number of flagged sequences remains stable until $n$ is reduced by more than $50\%$ from its calibrated value. For $50$-token suffixes, this stability extends across all tested values of $n$, with the absolute number of PA-memorized sequences remaining uniformly low (varying only between $4$ and $29$).

% Figure~\ref{fig:sensitivity} shows how the number of PA-memorized sequences changes with the relative beleif ratio threshold, $n$. We observe that for $4$-token suffixes, the number of PA-memorized sequences is generally stable until $n$ is over $50\%$ less than its calibrated value. (Note that we dynamically compute $\frac{P(s\vert p)}{
% \vshat}$ over ground-truth $s$ that
% are known to be easy-to-predict for LLMs for every model size.) For $50$ token suffixes, we similarly see stable values, this time at all values of $n$, since the minimum number varies from 4 to 29.

\begin{figure}[ht]
    \centering
    \begin{subfigure}[b]{0.7\linewidth}
        \centering
        \includegraphics[width=\linewidth]{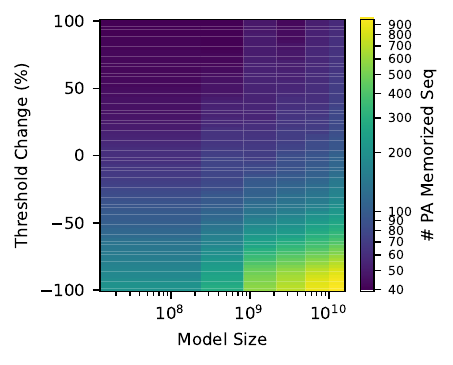}
        \caption{4 token suffixes}
        \label{fig:heatmap_4}
    \end{subfigure}
    \hfill
    \begin{subfigure}[b]{0.7\linewidth}
        \centering
        \includegraphics[width=\linewidth]{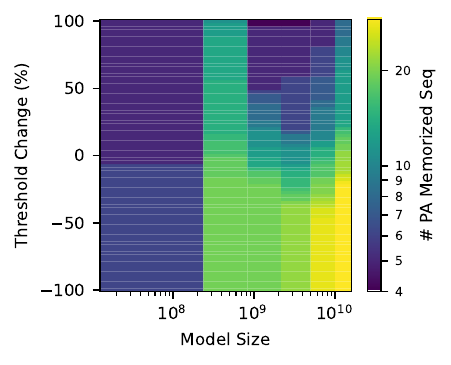}
        \caption{50 token suffixes}
        \label{fig:heatmap_50}
    \end{subfigure}
    \vspace{-0.1in}
    \caption{Sensitivity analysis of the threshold $n$ across various OPT model sizes for both $4$ token suffixes (a), and $50$ token suffixes (b). The heatmap displays the PA-memorization rate as the per-model threshold $n$ is varied from $-100\%$ to $100\%$ of its original value.}
    \label{fig:sensitivity}
    \vspace{-0.15in}
\end{figure}

% \begin{figure}
%     \centering
%     \includegraphics[width=0.7\linewidth]{images/counterfactual_heatmap_the_pile_4.pdf}
%     \includegraphics[width=0.7\linewidth]{images/counterfactual_heatmap_the_pile_50.pdf}
%     \vspace{-0.1in}
%     \caption{Sensitivity analysis of the threshold $n$ across various OPT model sizes for both $4$ token suffixes (Figure a), and $50$ token suffixes (Figure b). The heatmap displays the PA-memorization rate across as the per-model threshold $n$ is varied from $-100\%$ to $100\%$ of its original value.}
%     \label{fig:-approx}
%     \vspace{-0.15in}
% \end{figure}

% \begin{figure}
%     \centering
%     \includegraphics[width=0.7\linewidth]{images/counterfactual_heatmap_the_pile_50.pdf}
%     \vspace{-0.1in}
%     \caption{Counterfactual Memorization (x-axis) vs \name Memorization (y-axis). Each data point averages over models with the same frequency of exact matches and near duplicates.}
%     % We see that both metrics are positively correlated across different training settings
%     % Fix axes
%     \label{fig:counterfact-approx}
%     \vspace{-0.15in}
% \end{figure}

\section{LLM Usage}
LLMs were used for fundamental background research and polishing wording for grammar and clarity.

\end{document}